\documentclass[runningheads]{llncs}
\usepackage[utf8]{inputenc}
\usepackage{graphicx}
\usepackage{amsmath,amssymb}
\usepackage{array}
\newcolumntype{L}[1]{>{\raggedright\let\newline\\\arraybackslash\hspace{0pt}}m{#1}}
\usepackage{booktabs,multirow}
\usepackage{subcaption}
\usepackage{placeins}
\usepackage{adjustbox}
\usepackage[inline]{enumitem}
\usepackage{tabularx,ragged2e}
\newcolumntype{Y}{>{\centering\arraybackslash}X}
\newcolumntype{S}{>{\hsize=.4\hsize}X}
\newcolumntype{s}{>{\raggedleft\arraybackslash\hsize=.3\hsize}X}
\newcolumntype{M}{>{\RaggedRight\arraybackslash\hspace{0pt}}X}

\usepackage[pagebackref=true,breaklinks=true,colorlinks,bookmarks=false,pagebackref=false]{hyperref} 
\newcommand*\samethanks[1][\value{footnote}]{\footnotemark[#1]}
\usepackage{cleveref} 

\newlength{\foursubht}
\newsavebox{\foursubbox}

\begin{document}

\title{Optimizing the Dice Score and Jaccard Index for Medical Image Segmentation: Theory \& Practice}
\titlerunning{Optimizing Dice and Jaccard for Medical Image Segmentation}

\author{
    Jeroen Bertels\inst{1}\thanks{J.B. and T.E. have contributed equally to this work.} \and
    Tom Eelbode\inst{1}\samethanks \and
    Maxim Berman\inst{1} \and
    Dirk Vandermeulen\inst{1} \and
    Frederik Maes\inst{1} \and
    Raf Bisschops\inst{2} \and
    Matthew B.\ Blaschko\inst{1}
}

\authorrunning{J. Bertels, T. Eelbode et al.}

\institute{\vspace{1cm}}
\institute{
    ESAT, Center for Processing Speech and Images, KU Leuven, Belgium  \and
    Gastroenterology and hepatology, UZ Leuven, Belgium
}

\maketitle

\begin{abstract}

The Dice score and Jaccard index are commonly used metrics for the evaluation of segmentation tasks in medical imaging. Convolutional neural networks trained for image segmentation tasks are usually optimized for (weighted) cross-entropy. This introduces an adverse discrepancy between the learning optimization objective (the loss) and the end target metric. Recent works in computer vision have proposed soft surrogates to alleviate this discrepancy and directly optimize the desired metric, either through relaxations (soft-Dice, soft-Jaccard) or submodular optimization (Lovász-softmax). The aim of this study is two-fold. First, we investigate the theoretical differences in a risk minimization framework and question the existence of a weighted cross-entropy loss with weights theoretically optimized to surrogate Dice or Jaccard. Second, we empirically investigate the behavior of the aforementioned loss functions w.r.t.\ evaluation with Dice score and Jaccard index on five medical segmentation tasks. Through the application of relative approximation bounds, we show that all surrogates are equivalent up to a multiplicative factor, and that no optimal weighting of cross-entropy exists to approximate Dice or Jaccard measures. We validate these findings empirically and show that, while it is important to opt for one of the target metric surrogates rather than a cross-entropy-based loss, the choice of the surrogate does not make a statistical difference on a wide range of medical segmentation tasks. 


\end{abstract}

\section{Introduction}
The Dice score and Jaccard index have become some of the most popular performance metrics in medical image segmentation \cite{Kamnitsas2017,Ronneberger2015a,brats2018,isles2017,isles2018}.
Zijdenbos et al.\ were among the first to suggest the Dice score for medical image analysis by evaluating the quality of automated white matter lesion segmentations \cite{Zijdenbos1994}. 
In scenarios with large class imbalance, with an excessive number of (correctly classified) background voxels, they show that the Dice score is a special case of the kappa index, a chance-corrected measure of agreement. 
They further note that the Dice score reflects both size and localization agreement, more in line with perceptual quality compared to pixel-wise accuracy.

Risk minimization principle says we should minimize during training time the loss that we will be using to evaluate the performance at test time \cite{Vapnik:1995:NSL:211359}. 
This has motivated the introduction of differentiable approximations for Dice score (e.g. soft Dice \cite{Sudre2017}) and Jaccard index (e.g. soft Jaccard \cite{tarlow2012revisiting,nowozin2014optimal} or its more recent convex extension Lovász-softmax \cite{Berman2018a}) in order to incorporate it into gradient-based training schemes, such as stochastic gradient descent (SGD). These can be used for training segmentation models, including convolutional neural networks (CNNs)~\cite{Sudre2017}. 
Nevertheless, training with the pixel-wise cross-entropy loss, or its weighted variant, remains highly popular, even when the evaluation is performed using the Dice score or Jaccard index \cite{Kamnitsas2017,Chen2017}. In the MICCAI 2018 proceedings, 47 out of 77 learning-based segmentation papers used such a per-pixel loss even though the evaluation was performed with Dice score. 

This raises the question to what extent a loss function has impact on the prediction quality, and whether there are principled reasons for choosing one set of loss functions over another. 
In this work, we consider from a theoretical perspective the relationship between Dice score and Jaccard index, and work out that one approximates the other under risk minimization. 
We further question the existence of a well-weighted cross-entropy loss as a surrogate for Dice or Jaccard. We find an approximation bound between Dice and Jaccard losses, but no such approximation exists for cross-entropy. 
We are able to validate our findings empirically on five medical tasks, finding that all of the metric-sensitive losses are favourable over (weighted) cross-entropy, but that generally no mutual statistical difference can be observed among the former.

\section{Risk minimization with Dice and related similarities}\label{sec:2}

When performing discriminative training of machine learning methods, such as SGD for a CNN~\cite{Goodfellow-et-al-2016}, we are performing risk minimization. 
To learn a mapping $f$ from an observed input $x$ to a hidden variable $y$, empirical risk minimization optimizes the expectation of a loss function over a finite training set:
{
\setlength{\abovedisplayskip}{3pt}
\setlength{\belowdisplayskip}{5pt}
\begin{align}
\arg\min_{f\in\mathcal{F}} \underbrace{\frac{1}{n} \sum_{i=1}^n \ell(f(x_i),y_i)}_{=:\hat{\mathcal{R}}(f)} , 
\end{align}
}
where $\ell$ is a loss function and $\mathcal{F}$ is a function class of interest, e.g.\ the set of functions that can be represented by a neural network with a given topology.
We will denote the bootstrap distribution arising from a sample $\mathcal{S}:=\{(x_i,y_i)\}_{1\leq i \leq n}$ of size $n$ as $P_n$, and we may equivalently denote $\hat{\mathcal{R}}(f) = \mathbb{E}_{(x,y)\sim P_n}[\ell(x,y)]$.

In binary medical image segmentation, $y$ can be thought of as a set of pixels labeled as foreground. It is therefore well defined to consider set theoretic notions such as $y\cap \tilde{y}$ for two different segmentations.  This motivates the use of multiple set theoretic similarity measures  between two segmentations $y$ and $\tilde{y}$ including the Dice score $D$, the Jaccard index $J$, the Hamming similarity $H$, and what we will call the weighted Hamming similarity $H_\gamma$:
\begin{align}
D(y,\tilde{y}) := \frac{2 |y\cap \tilde{y}|}{|y| + |\tilde{y}|},\ 
J(y,\tilde{y}) := \frac{|y\cap \tilde{y}|}{|y\cup \tilde{y}|},\ 
H(y,\tilde{y}) := 1 - \frac{|y\setminus \tilde{y}| + |\tilde{y}\setminus y|}{d},\\ 
H_\gamma(y,\tilde{y}) := 1 - \gamma \frac{|y\setminus \tilde{y}|}{|y|} - (1 - \gamma) \frac{|\tilde{y}\setminus y|}{d - |y|} , \label{eq:hammingbound}
\end{align}
where $d$ denotes the number of pixels and $0\leq \gamma \leq 1$.  We note that all these similarities are between 0 and 1, and that $H_{\gamma}$ generalizes $H$ with equality when $\gamma = \frac{|y|}{d}$. A further important relationship is that between the Jaccard index and the Dice coefficient.  It is well known that 
\begin{align} \label{eq:DiceJaccardMonotonicRelationship}
J(y,\tilde{y}) = \frac{D(y,\tilde{y})}{2-D(y,\tilde{y})} \text{ and } D(y,\tilde{y}) = \frac{2 J(y,\tilde{y})}{1+J(y,\tilde{y})} .
\end{align}
Indeed, in the risk minimization framework for medical image segmentation, there are numerous examples where each of these measures are optimized \cite{pmid30557049,salehi2017tversky,Sudre2017}. 

In risk minimization, we replace a similarity $S : \mathcal{Y} \times \mathcal{Y} \rightarrow [0,1]$ with its corresponding loss $1-S$, and aim at minimizing this loss in expectation.
To train a neural network by backpropagation~\cite{Goodfellow-et-al-2016} it is necessary to replace this value with a differentiable surrogate. 
For the Hamming similarity, cross-entropy loss and other convex surrogates are statistically consistent~\cite{Bartlett2016JASA,Lapin_2016_CVPR}. 
To optimize the weighted Hamming similarity, one may employ weighted loss functions \cite{Ling2010} such as weighted cross entropy. 
Similarly, differentiable surrogates have been proposed both for the Dice score (e.g.\ soft Dice \cite{Sudre2017}) and Jaccard index (e.g.\ soft Jaccard \cite{SoftJaccard2016} and Lovász-softmax \cite{Berman2018a}).
Next, we hereby discuss the absolute and relative approximations between Dice and Jaccard and inspect the existence of an approximation through a weighted Hamming similarity.

\begin{definition}[Absolute approximation]
A similarity $S$ is absolutely approximated by $\tilde{S}$ with error $\varepsilon \geq 0$ if the following holds for all $y$ and $\tilde{y}$:
\begin{align}
    | S(y,\tilde{y}) - \tilde{S}(y,\tilde{y}) | \leq \varepsilon .
\end{align}
\end{definition}
\begin{definition}[Relative approximation]
A similarity $S$ is relatively approximated by $\tilde{S}$ with error $\varepsilon \geq 0$ if the following holds for all $y$ and $\tilde{y}$:
\begin{align}
    \frac{\tilde{S}(y,\tilde{y})}{1+\varepsilon} \leq S(y,\tilde{y}) \leq \tilde{S}(y,\tilde{y})(1+\varepsilon) .
\end{align}
\end{definition}
We note that both notions of approximation are symmetric in $S$ and $\tilde{S}$.

\begin{proposition}\label{thm:DiceJaccardApproximation}
$J$ and $D$ approximate each other with relative error of $1$ and absolute error of $3 - 2\sqrt{2}=0.17157\dots$. \end{proposition}
\begin{proof}
The relative error between $J$ and $D$ is given by (cf.\ Equation~\eqref{eq:DiceJaccardMonotonicRelationship})
\begin{align}
    \min_{\varepsilon\geq 0} \varepsilon ,\ 
    \text{s.t. }
    x \leq \frac{x}{2-x} (1+\varepsilon),  \ \forall \ 0\leq x \leq 1. \\
    x \leq \frac{x}{2-x}(1+\varepsilon)
    \implies 1-x \leq \varepsilon \implies \varepsilon = 1 .
\end{align}
The absolute error between $J$ and $D$ is given by
\begin{align}
   \varepsilon = \sup_{0\leq x \leq 1} \left| x - \frac{x}{2-x} \right| = 3 - 2\sqrt{2} ,
\end{align}
which can be verified straightforwardly by first order conditions:
\begin{align}
\frac{\partial}{\partial x} \left( x - \frac{x}{2-x} \right)
= 0 \implies (2-x)^2 - 2 = 0 \implies x =
2 - \sqrt{2}. \rlap{$\quad \qed$}
\end{align}
\end{proof}

\begin{proposition}
$D$ and $H_{\gamma}$ (where $\gamma$ is chosen to minimize the approximation factor between $D$ and $H_{\gamma}$) do not relatively approximate each other,  and absolutely approximate each other with an error of $1$. We note that the absolute error bound is trivial as $D$ and $H_{\gamma}$ are both similarities in the range $[0,1]$.
\end{proposition}
\begin{proof}
For relative error, consider the case that $|y\setminus \tilde{y}|=0$, $|\tilde{y}\setminus y| = \alpha d$, and $|y \cap \tilde{y}| = \alpha^2 d$ for some $0\leq \alpha<\frac{\sqrt{5}-1}{2}$:
\begin{align}
    \inf_{\gamma} \sup_{y, \tilde{y}} 1 - \gamma \frac{|y \setminus \tilde{y}|}{|y|} - (1-\gamma)\frac{|\tilde{y} \setminus y|}{d-|y|} - \frac{2|y \cap \tilde{y}|}{|y \triangle \tilde{y}| + 2|y \cap \tilde{y}|} (1+\varepsilon) \leq 0  \\
    \implies \sup_{0\leq\alpha<\frac{\sqrt{5}-1}{2}} 1 - \frac{\alpha }{1- \alpha^2 } -  \frac{2\alpha^2 }{\alpha  + 2\alpha^2 } (1+\varepsilon) \leq 0
\end{align}
If we let $\alpha \rightarrow 0$, it must be the case that $\varepsilon \rightarrow \infty$. To show that the absolute approximation error is 1, we similarly take
\begin{align}
    \lim_{\alpha\rightarrow 0} 1 - \frac{\alpha }{1- \alpha^2 } -  \frac{2\alpha }{1  + 2\alpha} = 1. \rlap{$\qquad \qed$}
\end{align}
\end{proof}

\begin{corollary}\label{thm:DiceHammingNoApproximation}
$D$ and $H$ do not relatively approximate each other, and absolutely approximate each other with an error of $1$.  
\end{corollary}
From these bounds, we see that a (weighted) binary loss can be an arbitrarily bad approximation for Dice when segmenting small objects, while the Jaccard loss gives multiplicative and additive approximation guarantees. Furthermore,  Eq.~\eqref{eq:DiceJaccardMonotonicRelationship} implies that $1-D(y,\tilde{y}) \leq 1-J(y,\tilde{y}) \implies \mathbb{E}_{(x,y)\sim P_n}[1-D(y,f(x))] \leq \mathbb{E}_{(x,y)\sim P_n}[1-J(y,f(x))]$ and optimization with risk computed with the Jaccard loss minimizes an upper bound on risk computed with the Dice loss.  Similarly setting $\varphi(x) =  2x/(1+x)$, by application of Jensen's inequality we arrive at $\mathbb{E}_{(x,y)\sim P_n}[1-J(y,f(x))] =\mathbb{E}_{(x,y)\sim P_n}[\varphi(1-D(y,f(x)))] \leq \varphi(\mathbb{E}_{(x,y)\sim P_n}[1-D(y,f(x))])$ and optimizing the Dice loss minimizes an upper bound on the Jaccard loss as $\varphi$ is a monotonic function over $[0,1]$.

\section{Empirical setup}
To test the aforementioned properties empirically, we investigate the performance of segmentation networks trained with different loss functions: cross-entropy (CE), weighted cross-entropy (wCE), soft Dice (sDice), soft Jaccard (sJaccard), and Lovász-sigmoid. 
We validate by cross-validation on five medical binary segmentation tasks. 
Three tasks are publicly available 3D datasets: BRATS 2018 (limited to whole tumor segmentation \cite{brats2018}; BR18, 285 images), ISLES 2017 (follow-up stroke lesion segmentation \cite{isles2017}; IS17, 43 images) and ISLES 2018 (acute stroke lesion segmentation \cite{isles2018}; IS18, 94 images). 
Furthermore, we expand the empirical setup with two in-house 2D datasets: lower-left third molar segmentation from panoramic dental radiographs (MO17, 400 images) and segmentation of colorectal polyps from colonoscopy images (PO18, 1166 images). \\
\textbf{Network architectures and preprocessing.}
For BR18, IS17 and IS18 we implement a U-Net-like~\cite{Ronneberger2015a} architecture with 3D convolutions, starting from a top-ranked implementation during last year’s BRATS challenge~\cite{Isensee2018} with less filters and an encoder depth of 7 layers. For MO17 the same architecture with 2D convolutions is used. For PO18 a VGG16 backbone architecture with atrous convolutions and pretrained on ImageNet~\cite{Chen2016} is used.
We use all image modalities available in each dataset as input, excluding perfusion data for IS17 and IS18. In order to fit memory, these inputs are resized and cropped. 
Data augmentation consisted of Gaussian noise, translations, flips, and in-plane rotations. \\
\textbf{Training procedure.}
We perform an initial training of the CNNs with cross-entropy loss. 
We use Adam~\cite{kingma2014adam} with an initial learning rate of $10^{-3}$ for the randomly initialized networks, and $10^{-4}$ for the ImageNet-initialized network. This learning rate is decreased when the validation loss stagnates. We stop the training when the validation loss starts increasing. Batch sizes are $40$ for MO17, $16$ for PO18, and 4 for all public datasets.
After initial convergence with cross-entropy, we continue training using one of the five different loss functions: CE, wCE, sDice, sJaccard and Lovász. 
For wCE, theory suggests that no optimal approximation w.r.t. Dice or Jaccard can be derived before training (see Sect. \ref{sec:2}). 
To set the weights, we therefore resort to the common heuristic of balancing foreground and background equally~\cite{Sudre2017}.
Thus, the weight applied to the foreground class is $1/(2p)$ and the weight applied to the background class is $1/(2-2p)$, with $p$ the foreground prior. 
We use the same optimization procedure as described for the initial training, with an initial learning rate of $10^{-3}$ for MO17 and $10^{-4}$ for all other datasets, which lead to appropriate convergence.

\section{Results and discussion}
\begin{table}[!htbp]
    \caption{Dice scores and Jaccard indexes obtained for each dataset with the different losses.
    Values in italic point to a significant lower result compared to each of the metric-sensitive losses. 
    Underlined values point to a significant lower result within the two groups of losses considered: the group of CE and wCE losses, and the group of metric-sensitive losses. 
    Values in bold point to a significant better result compared to all other losses. Values in parentheses are dataset sizes.}
    \begin{tabularx}{\linewidth}{lsXYYYYY}
    \toprule
         & Dataset & \multicolumn{1}{r}{\lapbox[\width]{1em}{\emph{loss} $\rightarrow$}} & CE &                wCE &               sDice &               sJaccard &               Lovász \\
    \midrule
\parbox[t]{7mm}{\centering\multirow{5}{*}{\rotatebox[origin=c]{90}{Dice score}}}
&    BR18  &  &  \textit{0.768} &  \underline{\textit{0.735}} &  0.823 &  0.823 &  \textbf{0.827} \\
&    IS17  &
    &   \underline{\textit{0.260}} &  0.311 &  0.331 &  0.321 &  0.305 \\
&    IS18  & &  \textit{0.463} &  \textit{0.474} &  \textbf{0.538} &  0.528 &  \underline{0.508} \\
&    MO17  & &   0.930 &   \underline{\textit{0.860}} &  0.932 &  0.931 &  0.932 \\
&    PO18  &  &  \textit{0.635} &  \underline{\textit{0.602}} &  0.656 &   0.651 &  0.649 \\
    \midrule
    \parbox[t]{7mm}{\centering\multirow{5}{*}{\rotatebox[origin=c]{90}{Jaccard index}}}
      & BR18  & &  \textit{0.654} &  \underline{\textit{0.602}} &  0.717 &   0.720 &  0.722 \\

        & IS17 &  &   \underline{\textit{0.177}} &   0.212 &  0.227 &  0.217 &  0.204 \\
         
        & IS18 &  &  \textit{0.345} &  \textit{0.344} &  \textbf{0.407} &   0.399 &  \underline{0.382} \\
        & MO17 & &  0.873 &  \underline{\textit{0.769}} & 0.877 &  0.875 &  0.877 \\
        & PO18 & &  \textit{0.541} &  \underline{\textit{0.488}} &  0.559 &  0.554 &  0.553 \\
    \bottomrule
\end{tabularx}
    \label{tab:results}
\end{table}
In the following discussion, we distinguish between two groups of losses. First, CE and wCE losses, which are surrogates for the (weighted) Hamming loss.
Second, sDice, sJaccard and Lovász losses, which are surrogates either for the Dice score or Jaccard index, and which we group as~\emph{metric-sensitive losses}. Table~\ref{tab:results} lists the average Dice scores and Jaccard indexes obtained after five-fold cross-validation for each dataset and loss under study. 
For each fold, we choose the best performing model w.r.t. the validation loss. 
We perform a pairwise non-parametric significance test (bootstrapping) with a p-value of 0.05 to assess inferiority or superiority between pairs of optimization methods. \\
\textbf{Equivalence of $J$ and $D$.} 
The theory suggests an equivalence between Dice and Jaccard metrics~(Prop.~\ref{thm:DiceJaccardApproximation}). 
This equivalence appears in our results: in particular, we found the rankings of the performance of the different losses to be the same in terms of Dice score and in terms of Jaccard index. \\
\textbf{Performance of the surrogates.} 
It is clear that CE and wCE lead to lower Dice scores and Jaccard indexes than the metric-sensitive losses (highlighted in italic). 
Only for MO17 does CE lead to similar performance compared to the metric-sensitive losses, likely due to a more uniform distribution of foreground and background pixels for this dataset.
This trend was expected due to the theoretical divergence between cross-entropy losses and the metric-sensitive losses and holds with current works optimizing Dice or Jaccard measures directly via their surrogates. Moreover, we found in general no statistically significant difference within the group of metric-sensitive losses w.r.t. Dice or Jaccard. 
This further confirms our theoretical findings and leaves the researcher a free choice. \\
\textbf{Weighting of cross-entropy.} 
We note that wCE is generally performing poorly compared to CE (inferior performances are underlined).
A better choice of weights might lead to a better performance of wCE.
However, it is clear from our results that the weighting is highly task-dependent.
Finding a better weighting is therefore non-trivial, compared to using one of the metric-sensitive losses. 
Moreover, as highlighted in our subsequent scale-specific study, wCE does yield a better performance within some restricted ranges of object scales. 
In accordance with theory, it is likely that no single weighting would yield appropriate surrogates to the target metrics across all datasets and scales. \\
\textbf{Scale-specific study.\label{par:scale-sensitive}}
In general, a segmentation dataset contains objects of variable size. 
It is generally assumed that Dice or Jaccard-sensitive losses have most impact for refining the segmentations of samples of small size, thanks to their invariance to scale, which cross-entropy does not have~\cite{Berman2018a}. 
The dependence of the approximation bound on the Hamming loss in Eq.~\eqref{eq:hammingbound} on the number of positive pixels in the ground truth $|y|$ also points towards a loss of segmentation accuracy in terms of Dice score in the small-sample regime when optimizing with Hamming loss or cross-entropy. 
We study this dependence in Fig.~\ref{fig:dice_size}, showing the average Dice scores as a function of the ground truth object size for the different optimization methods. 
We found that the applicability of metric-sensitive losses goes beyond the small-size regime, and that it is possible for CE to perform poorly across almost all scales. 
This is most evident in BR18 and IS18; even in the other datasets, the cross-entropy curve is dominated by other optimization methods. 
Furthermore, while wCE improves on CE on some datasets and area ranges, it can also vastly underperform the metric-sensitive losses as in BR18 and MO17, further indicating that a simple re-weighting of cross-entropy is not sufficient to capture the target metric across all object scales and datasets.
\begin{figure}[t]%
\sbox\foursubbox{%
  \resizebox{\textwidth}{!}{%
    \includegraphics[height=3cm]{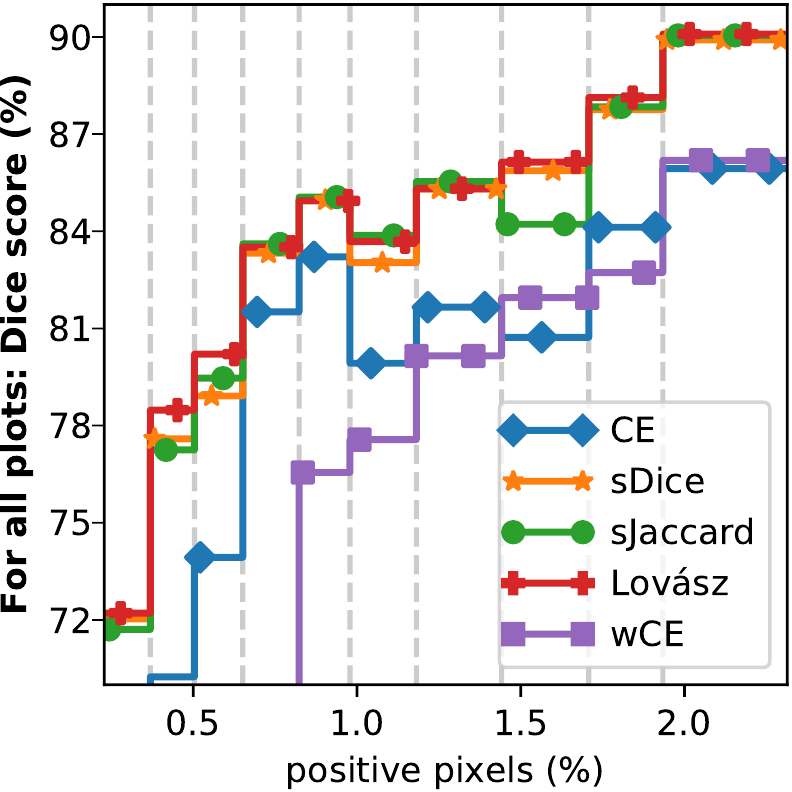}%
    \includegraphics[height=3cm]{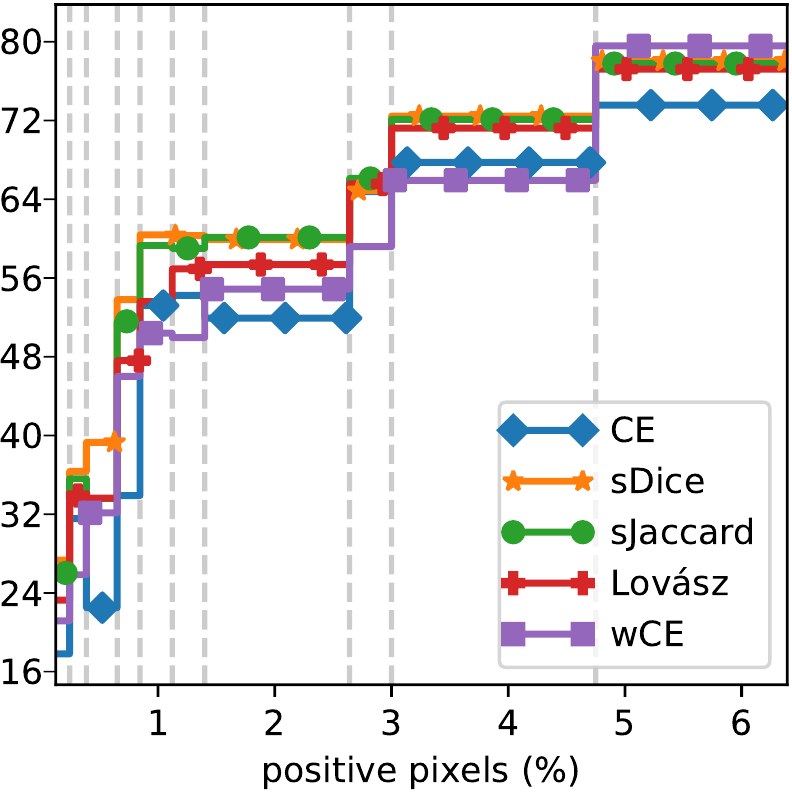}%
    \includegraphics[height=3cm]{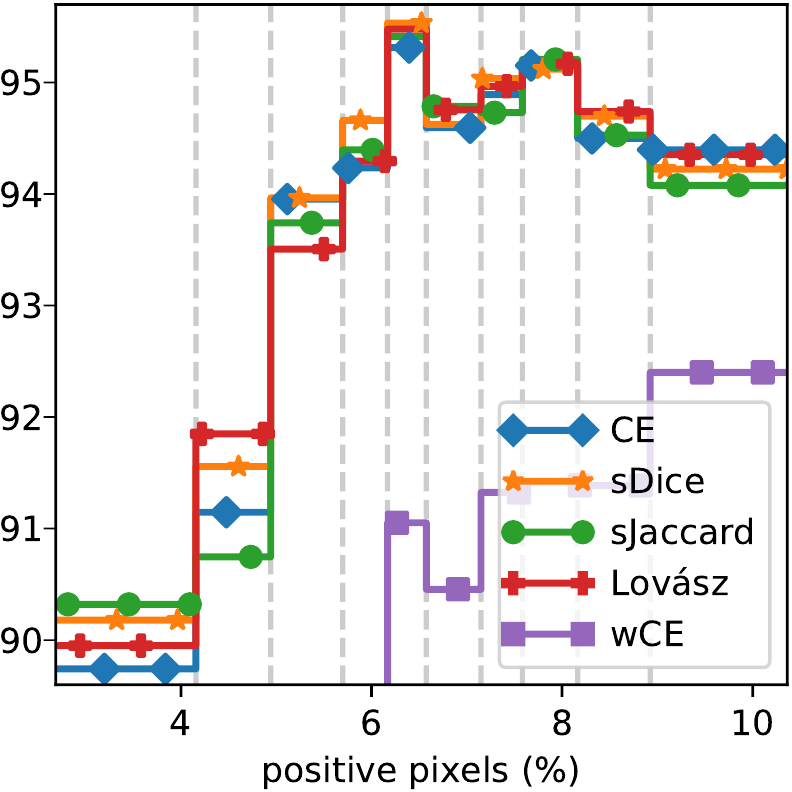}%
    \includegraphics[height=3cm]{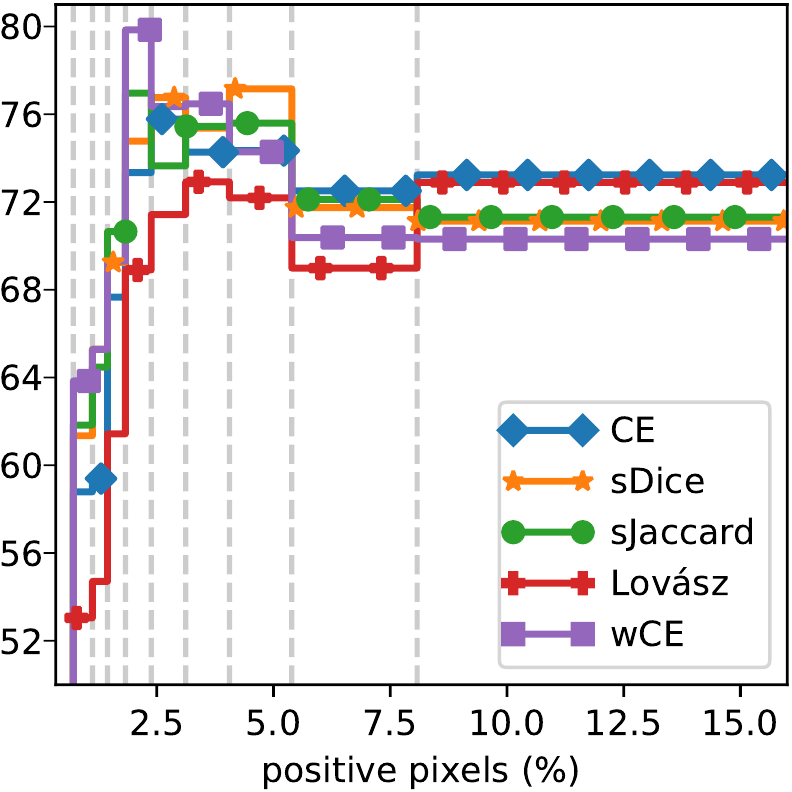}%
  }%
}%
\setlength{\foursubht}{\ht\foursubbox}%
  \centering%
\subcaptionbox{BRATS 2018\label{fig:areaBR}}{%
  \includegraphics[height=\foursubht]{histfig/BRATS_2018_decomp}%
}%
\subcaptionbox{ISLES 2018\label{fig:areaIS}}{%
  \includegraphics[height=\foursubht]{histfig/ISLES_2018_decomp}%
}%
\subcaptionbox{MO17\label{fig:areaMO}}{%
  \includegraphics[height=\foursubht]{histfig/38_decomp}%
}%
\subcaptionbox{PO18\label{fig:areaPO}}{%
  \includegraphics[height=\foursubht]{histfig/POLYPS_decomp}%
}%
  \caption{Dice score as a function of the relative ratio of foreground pixels for four datasets. 
  The scores are averaged within the 10 regions bordered by the dashed lines; each region contains 1/10 of the dataset. 
  Metric-sensitive losses perform as well or better than cross-entropy over most of the relative area ranges.
  ISLES 2017 omitted for lack of statistical relevance given its lower number of samples.\label{fig:dice_size}}%
\end{figure}

\section{Conclusion}
We compared optimization with five different loss functions from both theoretical and empirical perspectives. We find Jaccard and Dice approximate each other relatively and absolutely, while no approximation by a weighted Hamming similarity (i.e. a set theoretical equivalent for weighted cross-entropy) can be found. We confirm these findings empirically by evaluation on five medical segmentation tasks. We can show that there is generally no significant difference between the use of either of the metric-sensitive loss functions. Cross-entropy and its weighted version are however inferior to the latter when evaluated on Dice and Jaccard. This is in line with theory, which predicts that Jaccard controls the Dice loss. Nevertheless, the use of per-pixel losses remains highly popular. Of the 77 learning-based segmentation papers in the MICCAI 2018 proceedings that perform evaluation with Dice, 47 trained using a per-pixel loss.   The theory and empirical results presented here suggest that wider adoption of metric-sensitive losses like Dice and Jaccard is warranted. \\

\noindent
\textbf{Acknowledgements.}
This work is funded in part by Internal Funds KU Leuven (grant \# C24/18/047). The computational resources were partly provided by the Flemish Supercomputer Center (VSC). J.B. is part of NEXIS, a project that has received funding from the European Union's Horizon 2020 Research and Innovations Programme (grant \# 780026). R.B. is supported by FWO and Fujifilm. M.B.\ and M.B.B.\ acknowledge support from FWO (grant \# G0A2716N), an Amazon Research Award, an NVIDIA GPU grant, and the Facebook AI Research Partnership. The authors thank H. Willekens, C. Camps, C. Hassan, E. Coron, P. Bhandari, H. Neumann, O. Pech and A. Repici for their effort and collaboration.

\FloatBarrier

\bibliographystyle{splncs04}
\bibliography{biblio}

\end{document}